\documentclass{amsart}

\usepackage{amssymb,amsmath,amsthm}
\usepackage{color}
\makeindex
\usepackage{enumerate}
\usepackage{dsfont}
\usepackage{mathtools}
\usepackage{multirow}
\usepackage{graphicx}
\usepackage{booktabs}
\usepackage{makecell} % for more vertical space in cells
\setcellgapes{5pt}

\newtheorem{thm}{Theorem}
\newtheorem{theorem}[thm]{Theorem}
\newtheorem{corollary}[thm]{Corollary}
\newtheorem{lemma}[thm]{Lemma}
\newtheorem{proposition}[thm]{Proposition}

\theoremstyle{definition}
\newtheorem{remark}[thm]{Remark}
\newtheorem{definition}[thm]{Definition}
\theoremstyle{definition}

\DeclareMathOperator*{\argmin}{arg\,min}

\begin{document}

\title[]{Asymmetric Impurity Functions, Class Weighting, and Optimal Splits for Binary Classification Trees}
\author[]{David Zimmermann}

\iffalse
\address{Department of Mathematics\\
  University of California \\
  San Diego 92093}
\fi

\email{davidszimmermann@gmail.com}

\maketitle

\begin{abstract}
We investigate how asymmetrizing an impurity function affects the choice of optimal node splits when growing a decision tree for binary classification. In particular, we relax the usual axioms of an impurity function and show how
skewing an impurity function biases the optimal splits to isolate points of a particular class when splitting a node. We give a rigorous definition of this notion, then give a necessary and sufficient condition for such a bias
to hold. We also show that the technique of class weighting is equivalent to applying a specific transformation to the impurity function, and tie all these notions together for a class of impurity functions
that includes the entropy and Gini impurity. We also briefly discuss cost-insensitive impurity functions and give a characterization of such functions.   
\end{abstract}

\section{Introduction}

In supervised learning, decision trees and their related methods are among the most popular tools for classification. Their constructions are based on many criteria and parameters,
among them a chosen function to measure impurity of a node or dataset. This impurity function informs the optimal (greedy) split for a given node when growing the tree. (There are splitting criteria that are not based on impurity
functions, but we do not examine those here.) An impurity function satisfies certain axioms (which may slightly vary among different authors
and contexts), among them the property that the impurity function is symmetric in its entries. Intuitively, this condition says that an impurity function treats all classes equally during tree construction. For example, a dataset or tree node that consists
of 80\% Class 0 points and 20\% Class 1 points is equally ``impure" or of the same ``quality" as a tree node that consists of 20\% Class 0 points and 80\% Class 1 points. However, in many applications this is not necessarily desirable. A couple 
of contexts for which this may be the case:
\begin{itemize}
\item Imbalance in the number of occurrences of each class: If our dataset is highly imbalanced then detection of the rare class may be difficult. In this case one might, for example, consider an 80-20 mixture of points to be better
or more informative than a 20-80 mixture of the same size, depending on which class is the rare class.
\item Different costs for different misclassification types: The classic example of this is cancer detection, where the cost of a false negative is the death of a patient whereas the cost of a false positive (though often high) is not nearly as
catastrophic. In this case as well, the quality of an 80-20 mixture of points might be considered different from the quality of a 20-80 mixture of the same size.
\end{itemize}

Both of these situations arise frequently in practice, and the problem of dealing with them is well-studied \cite{HG08, JS02, LV13}. A common strategy that is used to deal with the first situation is oversampling or undersampling: one artificially increases the number of
samples of the rare class or decreases the number of samples of the common class in order to balance the prior class probabilities. There are many oversampling and undersampling techniques \cite{CBHK02,Ch09,HBGL08};
perhaps the simplest technique, which is the one we will consider in this paper, is class weighting: one simply scales the weights of all points of a chosen class by some fixed factor. A strategy that is used to deal with the second situation 
is to incorporate different misclassification costs into the impurity function itself \cite{BFOS84}. Along these lines, sensitivity of splitting criteria to different misclassification costs has been studied as well \cite{DH00, El01}.
Cost modification and class weighting are essentially just different perspectives on the same idea; for example, misclassifying a point of doubled weight incurs the same penalty as misclassifying an unweighted point with doubled misclassification cost. In this way, class
weighting can be thought of either as a simple over/undersampling technique or as a modification of misclassification costs.  

A different approach to dealing with imbalanced classes or misclassification costs is to choose an asymmetric impurity function to determine splits. Intuitively, the asymmetry in the impurity function should somehow naturally create a bias toward or against a particular class. Work by
Marcellin, Zighed, and Ritschard \cite{MZR06, MZR08} considered the case of imbalanced classes and proposed a family of asymmetric impurity functions. They showed a change in the shapes of the precision-recall and ROC curves for
several example datasets when using these asymmetric impurity functions in place of a symmetric function, giving an improvement in recall at lower-precision decision thresholds. The parametrized family $h_m: [0,1]\rightarrow\mathbb{R}$ they proposed is given by
\begin{equation}\label{eqn:MZRimp}
h_m(p) = \frac{p(1-p)}{(-2m+1)p+m^2},\qquad m\in(0,1)
\end{equation}
where the parameter $m$ is also the maximizer of $h_m$.

In this paper, we more closely investigate exactly how asymmetrizing an impurity function leads (at least locally) to favoring purity in one class over another when splitting a node. In particular, we relax the usual axioms of an impurity function (Definition \ref{def:preimp})
then compare two arbitrary impurity functions $f$ and $g$ and investigate what
causes $f$ to more strongly prefer purity in one class than $g$ does for a given split. We give a rigorous definition of this notion (Definitions \ref{def:equivalent}, \ref{def:positivelypure}), then state and prove a necessary and sufficient condition on $f$ and $g$
for such a comparison to hold (Theorem \ref{thm:mainthm}). We also show that class weighting is equivalent to applying a specific transformation to the impurity function (Definitions \ref{def:phi_w}, \ref{def:T_w}, Theorem \ref{thm:weightequiv}), and tie all of these preceding ideas together for
a class of impurity functions that includes the entropy and Gini impurity (Definition \ref{def:respectweight}, Theorem \ref{thm:sectionlink}). We also give a characterization of cost-insensitive impurity functions (Definition \ref{def:costinsensitive}, Theorem \ref{thm:insensitive}).
Along the way, we consider the typical axioms imposed upon an impurity function and remark on each axiom's utility and necessity.

This paper is organized as follows: in Section \ref{sec:conventions} we state some preliminary terminology, conventions, and notation. In Section \ref{sec:performance} we give motivation for our main definition and describe how certain performance
metrics relate to a single split of a node. In Section \ref{sec:mainsec} we give a modified definition of impurity function, then state and prove our main results about comparisons of impurity functions. In Section \ref{sec:weighting} 
we define a transformation on the set of impurity functions and show equivalence between this transformation and class weighting. We then relate this transformation back to Section \ref{sec:mainsec} and briefly discuss cost-insensitive impurity functions.
Finally, in Section \ref{sec:impaxioms} we close with a few remarks about the axioms of an impurity function as typically stated in the literature.

\section{Preliminary Terminology, Conventions, and Notation}\label{sec:conventions}

Throughout this paper, we only concern ourselves with binary classification; all underlying distributions of data are assumed to have two classes. We will refer to one of the
classes as {\it negatives} or {\it Class 0}, and to the other as {\it positives} or {\it Class 1}. We use the term {\it positive prevalence} of a tree node or dataset to refer to the weighted proportion of Class 1 points in said node or dataset. All trees are binary trees with each
non-leaf node having two nonempty children. Given a node that splits into two children,
we will refer to the child node with lower positive prevalence as the {\it left child}, and the child node with higher positive prevalence the {\it right child} (if both nodes have the same positive prevalence, label them as left and right arbitrarily). We will always use the letters
$c,a,b$ (sometimes subscripted) to denote the positive prevalences of the parent node, the left child, and the right child, respectively, and will refer to $a$ and $b$ as the {\it left} and {\it right positive prevalences}. We will use the letters $f$ and $g$ to denote impurity functions. Finally, for $a\leq c\leq b$ and a function $f$ we adopt the convention
$$
\left.\frac{b-c}{b-a}f(a) + \frac{c-a}{b-a}f(b)\right|_{(a,b)=(c,c)}=f(c).
$$

\section{Performance Metrics for a Single Split}\label{sec:performance}

In this section we provide some motivation and intuition for what follows in Sections \ref{sec:mainsec} and \ref{sec:weighting}. Let us begin with an example to illustrate the notion of ``preference for purity in a given class" for one impurity function versus another.

\begin{figure}
  \includegraphics[width=70 ex]{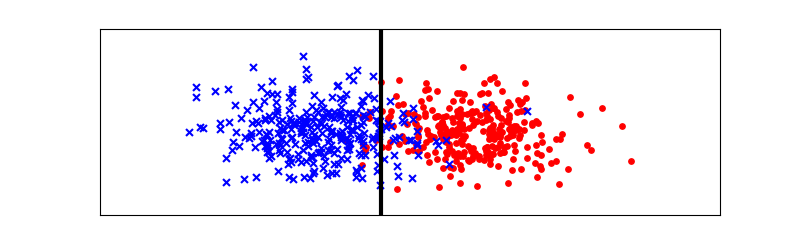}
  \includegraphics[width=70 ex]{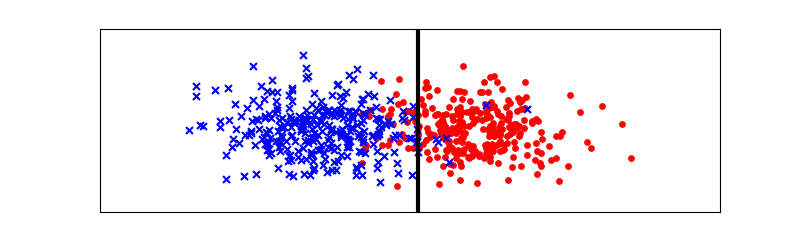}
  \includegraphics[width=70 ex]{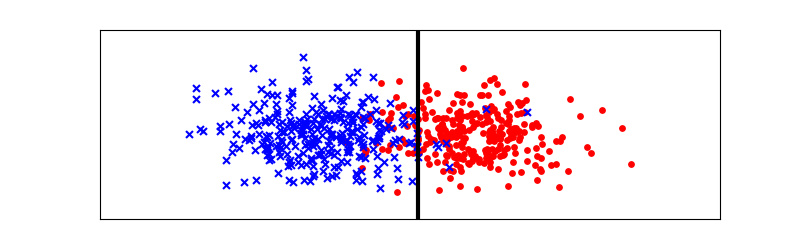}
  \includegraphics[width=70 ex]{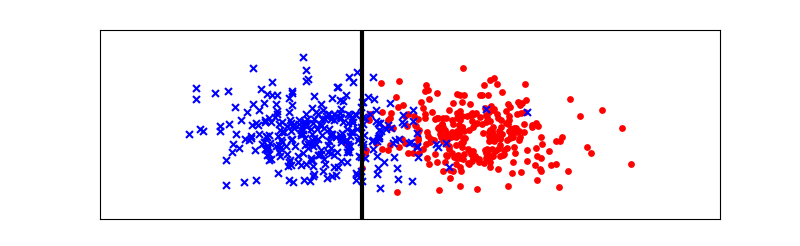}
  \caption{Top: A set of class 0 points (blue `x's) and Class 1 points (red circles) along with the optimal split with respect to the Gini impurity (bold black line). Second from top: Same set of points, but with optimal split with respect to the impurity function $f(p) = p-p^3$.
Third from top: Same points, but with the Class 1 points' weights halved. Weighted points are then split using Gini impurity. Bottom: Same points, but with the Class 1 points' weights scaled by a factor of 5. Weighted points are then split using Gini impurity.}
  \label{fig:splits}
\end{figure}

In Figure \ref{fig:splits}, the top plot shows a collection of Class 0 points (blue `x's) and Class 1 points (red circles)
in the plane, all of unit weight, along with the optimal single split of this set (bold black line) with respect to the Gini impurity $g(p) = 2p(1-p)$. In this plot we can see that the Gini impurity chooses a split that gives a left child that is quite pure (i.e., has a low positive prevalence) and a right child that is
also reasonably pure (i.e., has a high positive prevalence).

The second plot shows the same set of points, but now shows the optimal split with respect to the asymmetric impurity function $f(p)=p-p^3$.  In this plot we can see that this particular asymmetric impurity chooses a
split with a right child that is much more pure than the right child produced by the Gini impurity, but with the tradeoff of lower purity in the left child. Note also that the region corresponding to the right child is smaller. In this example,
the asymmetric impurity function $f$ preferred purity for Class 1 points more strongly than the Gini impurity did, whereas the Gini impurity preferred purity for Class 0 points more strongly than $f$ did. 

The third plot shows the same set of points, but now weighted so that all Class 1 points each have weight equal to 1/2 (which corresponds to undersampling Class 1 points). The split shown in this figure is the optimal split of this weighted set with respect to the Gini impurity.  
The Gini impurity on this weighted set shows similar behavior to the asymmetric impurity function, preferring purity for Class 1 points. Note that decreasing the weight of the Class 1 points increased the purity of the right child. This makes intuitive sense for the following reason: one can afford to
``pollute" the left child with Class 1 points without ruining the purity very much since the Class 1 points are light; on the other hand, polluting the right child with even a few Class 0 points can quickly ruin the purity since the Class 0 points are now relatively heavy.

The bottom plot shows the same set of points, but now weighted so that all Class 1 points each have weight equal to 5. The split shown in this figure is the optimal split of this weighted set with respect to the Gini impurity.  
The optimal split of this weighted set has the purest left child of all, with the least pure right child. Note also that the region corresponding to this right child is larger than in the other plots.

Now consider the following decision tree of depth 1 generated by a single split of a given dataset. Suppose our dataset has total weight equal to $W$ and has positive prevalence $c$. Suppose our single split yields children with positive prevalences $a < b$. Then the only nontrivial
classifier we can make from this tree is to classify the points in the left child as negatives and points in the right child as positives. Since the weights of the children are uniquely determined by their positive prevalences (see Proposition \ref{prop:childweights}), we therefore have
the following confusion matrix for this classifier:  
\makegapedcells
\begin{center}
\begin{tabular}{cc|cc}
\multicolumn{2}{c}{}
            &   \multicolumn{2}{c}{Predicted} \\
    &       &   Positive &   Negative              \\ 
    \cline{2-4}
\multirow{2}{*}{\rotatebox[origin=c]{90}{Actual}}
    & Positive      & $\frac{W}{b-a}(c-a) b$ & $\frac{W}{b-a}(b-c) a$                 \\
    & Negative    & $\frac{W}{b-a}(c-a)(1-b)$    & $\frac{W}{b-a}(b-c) (1-a)$                \\ 
    \cline{2-4}
\end{tabular}
\end{center}
\vspace{2ex}

Now we have the usual pairs of metrics to describe performance: true positive rate and false positive rate, and precision and recall. Another pair of metrics that describes classifier performance is {\it positive predictive value} (PPV)
and {\it negative predictive value} (NPV). The PPV is just a synonym for precision. The NPV is the analogue of precision for negative points; i.e., the NPV is the number of true negatives divided by the total number of predicted negatives. 
Now the true positive rate (recall) and false positive rate for the classifier above do not have a particularly nice form, but the PPV (precision) and NPV do: $PPV = b$, $NPV = 1-a$. In other words, a good split -- which tries to maximize $b$ and minimize $a$ --
tries to locally maximize PPV and NPV. In our example above, the asymmetric impurity function gave us a split with higher PPV than the split that the Gini impurity gave (on the unweighted set), with the
tradeoff of lower NPV. Weighting the Class 1 points instead by a factor 1/2 gave similar behavior. Equivalently, the Gini impurity on the unweighted set gave a split with higher NPV with the tradeoff of lower PPV. Weighting the Class 1 points by a factor of 5 gave an even higher NPV.

PPV and NPV are ``opposing" metrics in the sense that, loosely speaking, forcing an improvement in one metric typically leads to a worsening of the other metric, and vice versa. The same is true of precision and recall.
We will see in the next sections under what conditions an impurity function ``tries harder" to maximize PPV (precision) at the potential expense of NPV and recall, and vice versa.

\section{Comparison of Splitting Behavior for Different Impurity Functions}\label{sec:mainsec}

In much of the literature (e.g., the standard reference text \cite{BFOS84} by Breiman et al.) an impurity function is defined to be a function $f:[0,1]\rightarrow \mathbb{R}$ that satisfies three axioms:
\begin{enumerate} 
\item $f(p)$ is maximized only at $p=1/2$;
\item $f(p)$ is minimized only at the endpoints $p=0,1$;
\item $f$ is symmetric, i.e., $f(p)=f(1-p)$.
\end{enumerate}

It is also not uncommon to require (or implicitly assume) that $f$ satisfies other properties such as  concavity (often strict concavity), differentiability, and the condition that $f(0)=f(1)=0$. These variations in convention
are often minor, and most of the commonly used impurity functions in practice such as the entropy $f(p) = -p\log p - (1-p)\log(1-p)$ and the Gini impurity $f(p)=2p(1-p)$ satisfy all these properties anyway.

However, in this paper we relax most of the above properties. Let us now state the definition of impurity function that we will use throughout this paper.

\begin{definition}\label{def:preimp}
A {\it preimpurity function} is a function $f:[0,1]\rightarrow \mathbb{R}$ that satisfies the following two properties:
\begin{enumerate} 
\item $f$ is continuous on $[0,1]$ and $C^3$ on $(0,1)$;
\item $f''<0$ on $(0,1)$.
\end{enumerate}

If we also have $f(0)=f(1)=0$, then we call $f$ an {\it impurity function}.
\end{definition}

\begin{remark}
A couple remarks are worth making here: Firstly, the smoothness condition above, while stronger than what is typically imposed, will show to be a useful and convenient condition that facilitates the statements and proofs of the results throughout this section and the next.
We suspect that such smoothness is not actually necessary for our results to hold anyway (see Remark \ref{rem:nonsmoothness}). Concavity, on the other hand, is not only necessary to prove our results, but is also necessary in general to ensure
that an impurity function behaves well when splitting a node; we elaborate on this assertion in Section \ref{sec:impaxioms}. Again, most commonly used impurity functions, e.g. entropy and Gini impurity, satisfy these conditions as well.
(These conditions do exclude, for example, the misclassification rate $f(p)=\min(p,1-p)$ but that will not concern us.)

Secondly, despite the fact that we do not really care about the value of our impurity functions at the endpoints, we will see (Corollary \ref{cor:standardform}) that there is no loss of generality in fixing those values. We do want the flexibility of
allowing for arbitrary values at the endpoints, however, and will therefore be using preimpurity functions when discussing optimal splits.
\end{remark}

Recall the following basic facts about impurity of a node: The total impurity (with respect to a preimpurity function $f$) of a node $n$ with positive prevalence $c$ and total weight $W$ is $W\cdot f(c)$. If $n$ is split into two
children with positive prevalences $a$ and $b$ with $a\leq b$, then the combined total impurity of the children (which we will also refer to as the {\it impurity of the split}) is $W_l\cdot f(a) + W_r\cdot f(b)$, where $W_l,W_r$ are
the total weights of the points in the left child and right child, respectively. Now $W_l+W_r = W$. If $a=c=b$, then the children's combined total impurity simplifies to $W\cdot f(c)$ again. Otherwise, we have $a<c<b$. Now the total weight
of the Class 1 points in $n$ is $W c$. Then since we also have $Wc=W_l\, a + W_r\, b$ (since total weight of Class 1 points in $n$ is preserved) we can solve for $W_l,W_r$:
$$
W_l = W\cdot \frac{b-c}{b-a}, \qquad W_r = W\cdot \frac{c-a}{b-a},
$$
so that the total impurity of this split is
\begin{equation}\label{eqn:splitimp}
W\cdot\left(\frac{b-c}{b-a} f(a) + \frac{c-a}{b-a} f(b)\right).
\end{equation}
The optimal split with respect to $f$ is then the split whose left and right positive prevalences minimize (\ref{eqn:splitimp}).

We summarize the above observations as a proposition:

\begin{proposition}\label{prop:childweights}
Let $n$ be a node with positive prevalence $c$ and total weight $W$. If $n$ is split such that the left and right positive prevalences are equal to $a$ and $b$, respectively, then the weights $W_l,W_r$ of the left and right child are given by  
$$
W_l = W\cdot\frac{b-c}{b-a},\qquad W_r = W\cdot\frac{c-a}{b-a}
$$
and the total impurity of the split with respect to the preimpurity function $f$ is equal to
$$
W\cdot\left(\frac{b-c}{b-a} f(a) + \frac{c-a}{b-a} f(b)\right).
$$
\end{proposition}

We are now ready to start defining comparisons of preimpurity functions.

\begin{definition}\label{def:equivalent}
Let $f,g$ be preimpurity functions. We say $f$ {\it is equivalent to} $g$ if for every node $n$, 
and every set of possible splits of $n$, the optimal split (or splits) with respect to $f$ is the same as the optimal split with respect to $g$. In other words (see Remarks \ref{rem:onlytwo} and \ref{rem:allsplitspossible} below), $f$ is equivalent to $g$ if for 
all $c\in(0,1)$ and all finite subsets $S\subseteq([0,c)\times(c,1])\cup\{(c,c)\}$ we have
\begin{equation}\label{eq:equivalence}
\argmin_{(a,b)\in S}\left(\frac{b-c}{b-a}f(a) + \frac{c-a}{b-a}f(b)\right) = \argmin_{(a,b)\in S}\left(\frac{b-c}{b-a}g(a) + \frac{c-a}{b-a}g(b)\right).
\end{equation}
\end{definition}

\begin{remark}\label{rem:onlytwo}
In Definition \ref{def:equivalent} above, it suffices to only consider sets $S$ with two elements since the argmin of a function on a finite set can be determined by pairwise comparing
the values of the function over all possible pairs of inputs. It is also clear, though perhaps worth re-emphasizing, that Definition \ref{def:equivalent} does not use the minimum {\it values} of the expressions in (\ref{eq:equivalence}); only the {\it minimizers} matter since those are what
determine the splitting decision for a node. Hence we omit the total weight $W$ of $n$ in (\ref{eq:equivalence}). 
\end{remark}

\begin{remark}\label{rem:allsplitspossible}
Observe that every pair of possible splits of a node with positive prevalence $c$ yields two (possibly nondistinct) elements $(a_1,b_1), (a_2,b_2) \in ([0,c)\times(c,1])\cup\{(c,c)\}$. Conversely, every pair of (possibly nondistinct) elements
$(a_1,b_1),(a_2,b_2)\in([0,c)\times(c,1])\cup\{(c,c)\}$ is realizable as left and right positive prevalences of two splits of some dataset with positive prevalence $c$ (see Proposition \ref{prop:possiblesplits} below). Hence Equation (\ref{eq:equivalence})
above does indeed characterize splitting equivalence of preimpurity functions.
\end{remark}

\begin{proposition}\label{prop:possiblesplits}
Let $c\in(0,1)$, and let $(a_1,b_1),(a_2,b_2)\in([0,c)\times(c,1])\cup\{(c,c)\}$. Then there exists a dataset $D$ with positive prevalence $c$
such that: there exists two splits of $D$, one of which has left and right positive prevalences $a_1$ and $b_1$, and the other of which has left and right positive prevalences
$a_2$ and $b_2$.
\end{proposition}

\begin{proof}
Take $\mathbb{R}^2$ as a feature space. If $a_1<c<b_1$ and $a_2<c<b_2$, let
\begin{align*}
&R_1 = b_1a_2(c-a_1)(b_2-c)(1-c), &B_1 = (1-b_1)(1-a_2)(c-a_1)(b_2-c)c,\\
&R_2 = a_1a_2(b_1-c)(b_2-c)(1-c), &B_2 = (1-a_1)(1-a_2)(b_1-c)(b_2-c)c,\\
&R_3 = a_1b_2(b_1-c)(c-a_2)(1-c), &B_3 = (1-a_1)(1-b_2)(b_1-c)(c-a_2)c,\\ 
&R_4 = b_1b_2(c-a_1)(c-a_2)(1-c), &B_4 = (1-b_1)(1-b_2)(c-a_1)(c-a_2)c;
\end{align*} 
if $a_1<c<b_1$ and $a_2=c=b_2$, let
\begin{align*}
&R_1 = R_4 = b_1(c-a_1), &B_1 = B_4 = (1-b_1)(c-a_1),\\ 
&R_2 = R_3 = a_1(b_1-c), &B_2 = B_3 = (1-a_1)(b_1-c);
\end{align*} 
and if $a_1=a_2=c=b_1=b_2$, let
\begin{align*}
&R_1 = R_2 = R_3 = R_4 = c, &B_1 = B_2 = B_3 = B_4 = 1-c. 
\end{align*} 

For $i=1,2,3,4$, place a point of Class 1 with weight $R_i$ and a point of Class 0 with weight $B_i$ in the $i$th quadrant. Take $D$ to be the set of these points. A direct computation then shows that $D$ has positive prevalence $c$,
that the left and right half-planes have positive prevalences $a_1$ and $b_1$, respectively, and that the upper and lower half-planes have positive prevalences $a_2$ and $b_2$, respectively.  
\end{proof}

\begin{lemma}\label{lem:partialequivalence}
For every preimpurity function $f$ and every $A,B,C\in\mathbb{R}$ with $A> 0$ we have that $f$ is equivalent to the preimpurity function $\tilde{f}(p) = A\,f(p) + B p + C$. 
\end{lemma}

\begin{proof}
A direct computation shows that for every fixed $c\in(0,1)$ and every finite subset $S\subseteq([0,c)\times(c,1])\cup\{(c,c)\}$ we have
\begin{align*}
\argmin_{(a,b)\in S}\left(\frac{b-c}{b-a}\tilde{f}(a) + \frac{c-a}{b-a}\tilde{f}(b)\right) &= \argmin_{(a,b)\in S}\left(A\left(\frac{b-c}{b-a}f(a) + \frac{c-a}{b-a}f(b)\right) + B c + C\right)\\
&= \argmin_{(a,b)\in S}\left(\frac{b-c}{b-a}f(a) + \frac{c-a}{b-a}f(b)\right).
\end{align*} 
\end{proof}

\begin{definition}\label{def:positivelypure}
Let $f,g$ be preimpurity functions. We say $f$ {\it splits more positively purely} (or {\it more purely with respect to Class 1}) {\it than} $g$ if for every node $n$, 
and every set of possible splits of $n$, there exists an optimal split with respect to $f$ that produces a right child whose positive prevalence is greater than or equal to the
positive prevalence of every node produced by every optimal split of $n$ with respect to $g$. In other words,  $f$ splits more positively purely than $g$ if for 
all $c\in(0,1)$ and all finite subsets $S\subseteq[0,c)\times(c,1]$ we have
\begin{equation}\label{eq:positivelypure}
\max \left\{\argmin_{b : (a,b)\in S}\left(\frac{b-c}{b-a}f(a) + \frac{c-a}{b-a}f(b)\right)\right\} \geq \max \left\{\argmin_{b : (a,b)\in S}\left(\frac{b-c}{b-a}g(a) + \frac{c-a}{b-a}g(b)\right)\right\}.
\end{equation}\label{def:negativelypure}
Similarly, we say $g$ {\it splits more negatively purely} (or {\it more purely with respect to Class 0}) {\it than} $f$  if for every node $n$, 
and every set of possible splits of $n$, there exists an optimal split with respect to $g$ that produces a left child whose positive prevalence is less than or equal to the
positive prevalence of every node produced by every optimal split of $n$ with respect to $f$; i.e., $g$ splits more negatively purely than $f$ if for 
all $c\in(0,1)$ and all finite subsets $S\subseteq[0,c)\times(c,1]$ we have
\begin{equation}\label{eq:negativelypure}
\min \left\{\argmin_{a : (a,b)\in S}\left(\frac{b-c}{b-a}g(a) + \frac{c-a}{b-a}g(b)\right)\right\} \leq \min \left\{\argmin_{a : (a,b)\in S}\left(\frac{b-c}{b-a}f(a) + \frac{c-a}{b-a}f(b)\right)\right\}.
\end{equation}
\end{definition}

\begin{remark}\label{rem:pospure}
In (\ref{eq:positivelypure}), it again suffices to only consider sets $S$ with two elements since any finite $S$ can be reduced to the subset that contains the two elements that attain the left and right-hand sides
of (\ref{eq:positivelypure}). Furthermore, concavity of $f$ and $g$ imply that the pair $(c,c)$
is a maximizer of the expressions in (\ref{eq:positivelypure}). Since for every other pair $(a,b)$ we have $b>c$, the only way either side of
the inequality (\ref{eq:positivelypure}) can equal $c$ is if $S=\{(c,c)\}$, in which case (\ref{eq:positivelypure}) becomes trivial. (A similar argument holds for (\ref{eq:negativelypure}).) Hence, for convenience, we may exclude the pair $(c,c)$ from $S$ in Definition
\ref{def:positivelypure}.
\end{remark}

In light of our discussion in Section \ref{sec:performance}, Definition \ref{def:positivelypure} intuitively says that if $f$ splits more positively purely than $g$ then for any given node the optimal split with respect to $f$ has a higher PPV than the optimal split with respect to $g$.
This definition also assumes the convention that in case of ties, each of $f$ and $g$ chooses its optimal split with the highest right-child positive prevalence, hence the usage of $\max$ in (\ref{eq:positivelypure}). Analogous remarks hold when $g$ splits more negatively purely
than $f$. 

At this point, let us give a few examples to illustrate Definition \ref{def:positivelypure}. Let $f(p)=p-p^3$, $g(p)=2p(1-p)$, as we did with our example in Section \ref{sec:performance}. Then $f$ splits more positively purely than $g$, and $g$ splits more negatively
purely than $f$ (a fact that will become clear when we reach Theorem \ref{thm:mainthm}). Suppose we have a node of total weight equal to 1 and positive prevalence equal to 40\%, and suppose we have a choice of two possible splits: Split 1, which splits the
node into a left child with weight 0.4 and positive prevalence 10\%, and a right child with weight 0.6 and positive prevalence 60\%; and Split 2, which splits the node into a left child with weight 0.7 and positive prevalence 25\%, and a right child with weight 0.3 and
positive prevalence 75\%. We evaluate the impurities of Splits 1 and 2 with respect to $f$:
$$
\mbox{Split 1:}\quad 0.4\cdot f(0.10) + 0.6\cdot f(0.60) = 0.27, \qquad \quad \mbox{Split 2:}\quad 0.7\cdot f(0.25) + 0.3\cdot f(0.75) = 0.2625,
$$   
so Split 2 is the optimal split with respect to $f$. Now we evaluate the impurities of Splits 1 and 2 with respect to $g$:
$$
\mbox{Split 1:}\quad 0.4\cdot g(0.10) + 0.6\cdot g(0.60) = 0.36, \qquad \quad \mbox{Split 2:}\quad 0.7\cdot g(0.25) + 0.3\cdot g(0.75) = 0.375,
$$   
so Split 1 is optimal with respect to $g$. In this example we see $f$ preferred the split that had the highly pure right child while $g$ preferred the split with the highly pure left child.

A second example, one that illustrates Definition \ref{def:positivelypure} graphically, is given in Figure \ref{fig:pospuregraph}. Now for every impurity function $f$, every node $n$ of positive prevalence $c$ (and unit total weight), and every split of $n$ with left and right
positive prevalences equal to $a$ and $b$, the impurity of that split is equal to the $y$-value of the line segment between the points $(a,f(a))$ and $(b,f(b))$ at the point where $p=c$. In this example, let $f(p)=p-p^3$, $g(p)=2p(1-p)$ as before. Suppose we have a node of total weight
equal to 1 and a positive prevalence of  45\%. Suppose we have a choice of two splits: one split with left and right positive prevalences of 0\%  and 70\%; and the other split with left and right positive prevalences of 25\%  and 95\%. The top plot shows the graph of $f$
along with the line segments corresponding to our two splits. We can graphically see that the line segment for Split 2 lies below the line segment for Split 1 when $p=0.45$. So Split 2 has lower impurity, and is therefore optimal with respect to $f$.  The bottom plot shows the graph of $g$
along with the line segments corresponding to the same two splits. In this plot, we can see that the line segment for Split 1 lies below the line segment for Split 2 when $p=0.45$. So Split 1 has lower impurity, and is therefore optimal with respect to $g$. As with our previous example, we see $f$
preferred the split that had the highly pure right child while $g$ preferred the split with the highly pure left child.

\begin{figure}
  \includegraphics[width=70 ex]{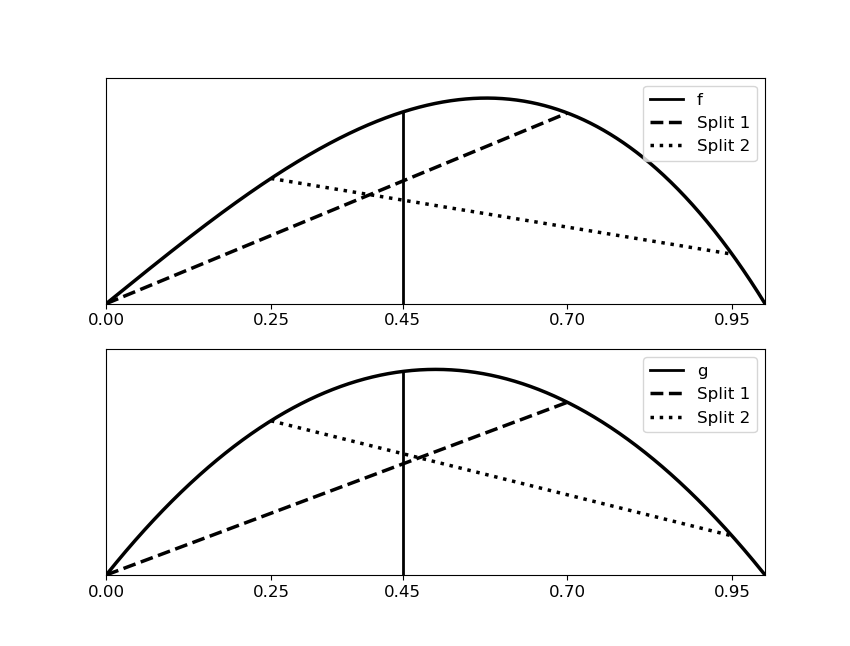}
  \caption{Top: Graph of impurity function $f(p)=p-p^3$, along with two possible splits. Bottom: Graph of impurity function $g(p)=2p(1-p)$, along with the same two splits.}
  \label{fig:pospuregraph}
\end{figure}

A third example, one that illustrates Definition \ref{def:positivelypure} on a dataset of points, is shown in the top two plots in Figure \ref{fig:splits} in Section \ref{sec:performance}. Here, the set of possible splits is all splits whose boundary is a vertical line.

Of course, for yet other examples, $f$ and $g$ might possibly choose the same split.

\begin{lemma}\label{lem:keylemma}
Let $f,g$ be preimpurity functions. Let $0\leq a_1<a_2<b_1<b_2\leq 1$ and suppose $f(a_1)=g(a_1)=f(b_1)=g(b_1)=0$. Suppose also that $f''/g''$ is increasing on $(a_1,b_2)$. Then
$$
\frac{f(a_2)}{g(a_2)} \leq \frac{f(b_2)}{g(b_2)} . 
$$ 
Furthermore, if $f''/g''$ is strictly increasing then the above conclusion is a strict inequality.
\end{lemma} 

\begin{proof}
Observe that the hypotheses and conclusion are invariant under scaling of $f$ and $g$ by positive constants, and observe that strict concavity implies that $f'(b_1)$ and $g'(b_1)$ are both negative. So we may also
suppose without loss of generality that $f'(b_1)=g'(b_1)$. Let $h = f''/g''$ so that $f'' = hg''$, and let $k=g-f$. Then $k(a_1)=k(b_1)=k'(b_1)=0$, and $k''=(1-h)g''$. 

Claim: $k\geq0$ on $[a_1,b_2]$.

Proof of claim: By Rolle's Theorem applied to $k$, there exists a $c\in(a_1,b_1)$ such that $k'(c)=0$. Now by Rolle's Theorem applied to $k'$, there exists a $d\in(c,b_1)$ such that $k''(d)=0$. Since $h$ is increasing and $g''<0$, we
therefore have that $k''\leq 0$ on $(a_1,d)$ and $k''\geq 0$ on $(d,b_2)$. So $k'$ is decreasing on $(a_1,d)$ and increasing on $(d,b_2)$. Since $k'(c)=0$, we have that $k'\geq 0$ on $(a_1,c)$ and $k'\leq 0$ on $(c,d)$; and
since $k'(b_1)=0$, we have that $k'\leq 0$ on $(d,b_1)$ and $k'\geq 0$ on $(b_1,b_2)$. This implies that $k$ is increasing on $[a_1,c]$, decreasing on $[c,b_1]$, and increasing on $[b_1,b_2]$. Finally, since 
$k(a_1)=k(b_1)=0$, we therefore conclude that $k\geq 0$ on $[a_1,b_2]$, proving the claim.

Now the above claim shows that both $k(a_2), k(b_2)\geq 0$, so that $g(a_2)\geq f(a_2)$ and $g(b_2)\geq f(b_2)$. Strict concavity of $g$ implies that $g(a_2)>0$ and $g(b_2) <0$, so that
$$
\frac{f(a_2)}{g(a_2)} \leq 1 \quad \mbox{and}\quad \frac{f(b_2)}{g(b_2)} \geq 1,
$$
and the desired result follows.

A straightforward modification of the above proof gives that our desired inequality is strict if $f''/g''$ is strictly increasing; details are omitted.
\end{proof}

We now present the main theorem of this section.

\begin{theorem}\label{thm:mainthm}
Let $f,g$ be preimpurity functions. Then $f$ splits more positively purely than $g$ if and only if $f''/g''$ is increasing on $(0,1)$.
\end{theorem}

\begin{proof}
$(\Leftarrow)$ Suppose $f''/g''$ is increasing. Fix $c\in(0,1)$, and let S be a finite subset of $[0,c)\times(c,1]$. In light of Remark \ref{rem:pospure} we may suppose $|S|=2$. 

Let $(a_1,b_1),(a_2,b_2)$ be the two elements of $S$, and suppose without loss of generality that $b_1 < b_2$ (if $b_1 = b_2$ then we immediately have equality in
Definition \ref{def:positivelypure} and we are done). So
$a_1,a_2 < c < b_1 <b_2$. We therefore want to show that if $(a_2,b_2)$ is the better
of the two splits with respect to $g$, then $(a_2,b_2)$ is also the better of the two splits with respect to $f$.
More precisely, we want to show that if
\begin{equation}\label{eq:gbetter}
\frac{b_2-c}{b_2-a_2}g(a_2) + \frac{c-a_2}{b_2-a_2}g(b_2) \leq \frac{b_1-c}{b_1-a_1}g(a_1) + \frac{c-a_1}{b_1-a_1}g(b_1) \mbox{ }
\end{equation}
then
\begin{equation}\label{eq:fbetter}
\frac{b_2-c}{b_2-a_2}f(a_2) + \frac{c-a_2}{b_2-a_2}f(b_2) \leq \frac{b_1-c}{b_1-a_1}f(a_1) + \frac{c-a_1}{b_1-a_1}f(b_1).
\end{equation} 
By Lemma \ref{lem:partialequivalence}, we may suppose without loss of generality that $f(a_1)=f(b_1)=g(a_1)=g(b_1)=0$. The above implication then reduces to
\begin{equation}\label{eq:goal}
\frac{b_2-c}{b_2-a_2}g(a_2) + \frac{c-a_2}{b_2-a_2}g(b_2) \leq 0 \quad \Rightarrow \quad \frac{b_2-c}{b_2-a_2}f(a_2) + \frac{c-a_2}{b_2-a_2}f(b_2) \leq 0.
\end{equation}
Strict concavity of $f$ and $g$ together with the fact that $b_1 < b_2$ implies $f(b_2)<0$ and $g(b_2)<0$. If $a_2 \leq a_1$, then $f(a_2)\leq 0$ and the right side of (\ref{eq:goal}) above is satisfied.
So suppose $a_2 > a_1$, so that $f(a_2) > 0$ and $g(a_2) > 0$. Rearranging the inequalities in (\ref{eq:goal}), we get that our desired condition is equivalent to
\begin{equation}\label{eq:equivgoal}
\frac{b_2 g(a_2)-a_2g(b_2)}{g(a_2)-g(b_2)} \leq c \quad \Rightarrow \quad \frac{b_2 f(a_2)-a_2f(b_2)}{f(a_2)-f(b_2)} \leq c.
\end{equation}
It is therefore sufficient to show 
\begin{equation}\label{eq:suffcond}
\frac{b_2 f(a_2)-a_2f(b_2)}{f(a_2)-f(b_2)} \leq \frac{b_2 g(a_2)-a_2g(b_2)}{g(a_2)-g(b_2)}.
\end{equation}
Clearing denominators and simplifying shows that (\ref{eq:suffcond}) is equivalent to
$$
f(b_2)g(a_2)\leq g(b_2)f(a_2),
$$
i.e.,
$$
\frac{f(a_2)}{g(a_2)} \leq \frac{f(b_2)}{g(b_2)}.
$$
But this follows from Lemma \ref{lem:keylemma}, and the desired conclusion follows.

$(\Rightarrow)$ Suppose that $f''/g''$ is not increasing. Since $f,g$ are $C^3$ and have nonvanishing second derivatives, $f''/g''$ is $C^1$. Hence there exists some interval $(a,b)\subseteq[0,1]$ such that $f''/g''$ is strictly decreasing
on $(a,b)$, i.e., $g''/f''$ is strictly increasing on $(a,b)$. 
Choose $a_1,a_2,b_1,b_2$ such that $a\leq a_1<a_2<b_1<b_2\leq b$. By Lemma \ref{lem:partialequivalence}, we may assume without loss of generality that $f(a_1)=f(b_1)=g(a_1)=g(b_1)=0$, so that
$f(a_2),g(a_2)>0$ and $f(b_2),g(b_2)<0$. Then by Lemma \ref{lem:keylemma} (reversing the roles of $f$ and $g$) we have
$$
\frac{g(a_2)}{f(a_2)}< \frac{g(b_2)}{f(b_2)}.
$$
A bit of algebra shows that the above inequality is equivalent to  
\begin{equation}\label{eq:reversesuffcond}
 \frac{b_2 g(a_2)-a_2g(b_2)}{g(a_2)-g(b_2)} < \frac{b_2 f(a_2)-a_2f(b_2)}{f(a_2)-f(b_2)}.
\end{equation}
Choose a $c$ such that
\begin{equation}\label{eq:choosec}
 \frac{b_2 g(a_2)-a_2g(b_2)}{g(a_2)-g(b_2)} < c < \frac{b_2 f(a_2)-a_2f(b_2)}{f(a_2)-f(b_2)}.
\end{equation}
Now
$$
a_2 \leq a_2 + \frac{(b_2-a_2)g(a_2)}{g(a_2)-g(b_2)} = \frac{b_2 g(a_2)-a_2g(b_2)}{g(a_2)-g(b_2)}<c
$$
by (\ref{eq:choosec}). Also, writing 
$$
b_1 = \frac{b_2-b_1}{b_2-a_2}\cdot a_2 + \frac{b_1-a_2}{b_2-a_2}\cdot b_2
$$
and using concavity of $f$, we get
$$
0 = f(b_1) \geq \frac{b_2-b_1}{b_2-a_2}f(a_2) + \frac{b_1-a_2}{b_2-a_2}f(b_2)
$$
which simplifies to
$$
\frac{b_2 f(a_2)-a_2f(b_2)}{f(a_2)-f(b_2)} \leq b_1
$$
so that $c<b_1$ by (\ref{eq:choosec}). We therefore have
$$
a_1<a_2<c<b_1<b_2
$$
with 
$$
 \frac{b_2 g(a_2)-a_2g(b_2)}{g(a_2)-g(b_2)} < c \qquad \mbox{and}\qquad c< \frac{b_2 f(a_2)-a_2f(b_2)}{f(a_2)-f(b_2)}
$$
which rearranges to

\begin{equation}\label{eq:counterex}
\frac{b_2-c}{b_2-a_2}g(a_2) + \frac{c-a_2}{b_2-a_2}g(b_2) < 0 \quad \mbox{and} \quad \frac{b_2-c}{b_2-a_2}f(a_2) + \frac{c-a_2}{b_2-a_2}f(b_2) > 0.
\end{equation}
Recalling that $f(a_1)=f(b_1)=g(a_1)=g(b_1)=0$, we have that (\ref{eq:counterex}) becomes

\begin{align*}
\frac{b_2-c}{b_2-a_2}g(a_2) + \frac{c-a_2}{b_2-a_2}g(b_2) &< \frac{b_1-c}{b_1-a_1}g(a_1) + \frac{c-a_1}{b_1-a_1}g(b_1) \\
\mbox{and}\\
\frac{b_2-c}{b_2-a_2}f(a_2) + \frac{c-a_2}{b_2-a_2}f(b_2) &> \frac{b_1-c}{b_1-a_1}f(a_1) + \frac{c-a_1}{b_1-a_1}f(b_1).\\
\end{align*}
Taking $S = \{(a_1,b_1),(a_2,b_2)\}$, we therefore have
$$
\argmin_{b : (a,b)\in S}\left(\frac{b-c}{b-a}f(a) + \frac{c-a}{b-a}f(b)\right) = b_1 < b_2  = \argmin_{b : (a,b)\in S}\left(\frac{b-c}{b-a}g(a) + \frac{c-a}{b-a}g(b)\right)
$$
so that $f$ does not split more positively purely than $g$.
\end{proof}

Theorem \ref{thm:mainthm} has a corresponding analogue, stated below, for one preimpurity function splitting more negatively purely than another; the proof is very similar and hence omitted.

\begin{theorem}\label{thm:negpure}
Let $f,g$ be preimpurity functions. Then $g$ splits more negatively purely than $f$ if and only if $g''/f''$ is decreasing on $(0,1)$.
\end{theorem}

Theorems \ref{thm:mainthm} and \ref{thm:negpure} immediately establish the relationship between splitting more positively purely and splitting more negatively purely:

\begin{corollary}\label{cor:posneg}
Let $f,g$ be preimpurity functions. Then $f$ splits more positively purely than $g$ if and only if $g$ splits more negatively purely than $f$.
\end{corollary}

\begin{remark}
In Definition \ref{def:positivelypure}, in (\ref{eq:positivelypure}) we broke ties by using $\max$ (i.e., by choosing the optimal split with highest right-child positive prevalence). In fact, we just as well
could have broken ties by using $\min$, and Theorem \ref{thm:mainthm} would still hold; the only modification necessary to the proof would be to replace all inequalities in (\ref{eq:gbetter}),(\ref{eq:fbetter}),(\ref{eq:goal}), and (\ref{eq:equivgoal}) with strict inequalities.
A similar remark of course holds for (\ref{eq:negativelypure}).   
\end{remark}

Corollary \ref{cor:posneg} implies a special case of the following general fact, alluded to in Section \ref{sec:performance} when discussing PPV versus NPV: an impurity function cannot produce an optimal split with both a higher right-child positive prevalence and a lower left-child positive
prevalence than an optimal split produced by another impurity function (assuming, of course, that both impurity functions are optimizing over the same set of splits). In other words, to improve purity in one class, one must sacrifice purity in the other class.
Proposition \ref{prop:puritytradeoff} makes this precise.

\begin{proposition}\label{prop:puritytradeoff}
Let $f,g$ be preimpurity functions, and suppose that $\{(a_1,b_1),(a_2,b_2)\}$ is the set of possible splits of some node with positive prevalence $c$. Suppose further that $(a_1,b_1)$ is optimal for $g$, and $(a_2,b_2)$ is optimal for $f$.
If $b_2>b_1$, then $a_2>a_1$.   
\end{proposition}

\begin{proof}
Suppose for contradiction that $a_2\leq a_1$. By Lemma \ref{lem:partialequivalence}, we may suppose without loss of generality that $g(a_2)=g(b_2)=0$. Then since $a_2\leq a_1 \leq c \leq b_2$, we have $g(a_1)\geq 0$.
Since $b_2>b_1\geq c$, we must also have $a_2<c$ so that $a_2<c\leq b_1<b_2$, giving $g(b_1)>0$ and $g(c)>0$. Then
$$
\frac{b_1-c}{b_1-a_1}g(a_1) + \frac{c-a_1}{b_1-a_1}g(b_1)>0=\frac{b_2-c}{b_2-a_2}g(a_2) + \frac{c-a_2}{b_2-a_2}g(b_2),
$$
so that $(a_1,b_1)$ is not optimal with respect to $g$, a contradiction.
\end{proof}

\begin{remark}
For any split of a node with unit weight we can use the Fundamental Theorem of Calculus and integration by parts to write the total reduction in impurity with respect to $f$ as
\begin{equation}\label{eq:reductionintegrals}
f(c)-\left(\frac{b-c}{b-a} f(a) + \frac{c-a}{b-a} f(b)\right) = \frac{b-c}{b-a}\int_a^c -f''(t)(t-a) \,dt + \frac{c-a}{b-a}\int_c^b -f''(t)(b-t)\,dt.
\end{equation}
From this equation we make a few observations: Firstly, the reduction in impurity depends only on $f''$ and not on the initial values of $f$ or $f'$. This is essentially a restatement of Lemma \ref{lem:partialequivalence}. Secondly, the
right hand side of (\ref{eq:reductionintegrals}) roughly tells us that if the mass of $-f''$ concentrates more to the right side of the unit interval than does the mass
of some other function $-g''$, then an increase in $b$ gives a proportionally larger reduction in impurity with respect to $f$ than with respect to $g$. This is a loose restatement of the backward implication
in Theorem \ref{thm:mainthm}. In general, one achieves a greater reduction in impurity with respect to $f$ by capturing a larger proportion of the mass under $-f''$ between $a$ and $b$, or by making $a$ and $b$
farther away from $c$.
 
\end{remark}

\begin{remark}\label{rem:nonsmoothness}
We suspect Theorem \ref{thm:mainthm} holds in more generality. In particular, suppose $f$ and $g$ are only assumed to be continuous and concave, but not necessarily differentiable or strictly concave.
Then $f''$ and $g''$ exist in the distributional sense as non-positive measures \cite{Sc66}. We then conjecture that $f$ splits more positively purely than $g$ if and only if $f''$ 
is absolutely continuous with respect to $g''$ and the Radon-Nikodym derivative of $f''$ with respect to $g''$ is increasing. Because the proof of this claim (if true) would likely be more involved than
the proofs of Lemma \ref{lem:keylemma} and Theorem \ref{thm:mainthm} without offering much additional insight into the nature of Definition \ref{def:positivelypure}, we do not pursue it.
\end{remark}

Theorem \ref{thm:mainthm} immediately gives us a few corollaries regarding equivalence of preimpurity and impurity functions.

\begin{corollary}\label{cor:equivalence}
Let $f,g$ be preimpurity functions. Then the following are equivalent:
\begin{enumerate} 
\item $f$ is equivalent to $g$.
\item $f''=Ag''$ for some constant $A>0$.
\item There exist constants $A,B,C\in\mathbb{R}$ with $A> 0$ such that $f(x) = Ag(x)+Bx+C.$
\end{enumerate}
\end{corollary}

\begin{proof}
$(1)\Rightarrow(2)$ Suppose $f$ and $g$ are equivalent. Then $f$ splits more positively purely than $g$, and vice versa. So both $f''/g''$ and $g''/f''$ are increasing by Theorem \ref{thm:mainthm}.
So $f''/g''$ is constant and, by strict concavity of $f$ and $g$, positive. So $f''=Ag''$ for some positive $A$.
 
$(2)\Rightarrow(3)$ This follows from the Fundamental Theorem of Calculus.

$(3)\Rightarrow(1)$ This is Lemma \ref{lem:partialequivalence}.
\end{proof}

\begin{corollary}\label{cor:standardform}
Let $f$ be a preimpurity function. Then there exists a unique (up to positive constant scaling) impurity function $\tilde{f}$ such that $f$ is equivalent to $\tilde{f}$.
\end{corollary}

\begin{proof}
Let $\tilde{f}(x) = f(x) + (f(0) - f(1))x - f(0)$. Then $\tilde{f}$ is an impurity function, and is equivalent to $f$ by Lemma \ref{lem:partialequivalence}.

To establish uniqueness, suppose $\tilde{f}_1$ and $\tilde{f}_2$ are impurity functions equivalent to $f$. Then they are equivalent to each other. So by Corollary \ref{cor:equivalence},
$\tilde{f}_1(x) = A\tilde{f}_2(x) + Bx +C$ for some $A,B,C \in\mathbb{R}$, $A>0$. The boundary conditions $\tilde{f}_1(0)=\tilde{f}_1(1)= \tilde{f}_2(0)=\tilde{f}_2(1)=0$ imply $B=C=0$,
so $\tilde{f}_1=A \tilde{f}_2$. 
\end{proof}

\begin{corollary}
Let $f,g$ be impurity functions. Then $f$ is equivalent to $g$ if and only if $f=Ag$ for some constant $A>0$.
\end{corollary}

\begin{proof}
This follows from Corollary \ref{cor:standardform}.
\end{proof}

%\begin{remark}\label{rem:MZRfamily}
Recall the family $h_m$ of impurity functions in (\ref{eqn:MZRimp}) given in the introduction. In light of Theorem \ref{thm:mainthm}, a direct computation shows that $h_{m_1}$ splits more positively
purely than $h_{m_2}$ if and only if $m_1\geq m_2$. (We will revisit this family in more detail in the next section.) For this particular family, moving the ``hump" (i.e. maximizer) of the function to the right is equivalent
to making the function split more positively purely. The next corollary shows that for arbitrary impurity functions, this is partially the case.
%\end{remark}

\begin{corollary}\label{cor:movehump}
Let $f,g$ be impurity functions, and suppose $f$ splits more positively purely than $g$. Then the maximizer of $f$ is greater than or equal to the maximizer of $g$.
\end{corollary}

\begin{proof}

Let $m_f,m_g\in(0,1)$ be the maximizers of $f$ and $g$, respectively (these maximizers are unique by strict concavity). Scaling $f,g$ by positive contants, we may assume without loss of generality that $g(m_g) = f(m_g) = 1$. 
Let $k = g-f$, so $k(m_f)\leq 0$ and $k(m_g)\geq 0$. If $k(m_g)= 0$ then $g(m_g) = f(m_g) = 1$ so $m_g$
is also the maximizer of $f$ and hence $m_f = m_g$ and we are done. So suppose $k(m_g) > 0$.

Claim: $k>0$ on $(0,m_g)$. 

Proof: Suppose for contradiction that $k(x_0) \leq 0$ for some $x_0\in(0,m_g)$. By Theorem \ref{thm:mainthm}, there exists an increasing $h$ such that $f'' = hg''$, so $k''=(1-h)g''$.
Now $k(m_f)\leq 0$ and $k(m_g)\geq 0$, so  by the Intermediate Value Theorem there exists some $c$ between $m_f$ and $m_g$ such that $k(c)=0$. In particular, $c\in(0,1)$, so
$k$ has at least three zeroes (since also $k(0)=k(1)=0$). Applying Rolle's Theorem to $k$ and $k'$, we then get that $k'$ has at least two zeroes, and $k''$ has at least one zero $d$.
Since $h$ is increasing and $g''<0$, we have that $h(d) =1$ and therefore
\begin{equation}\label{eq:eventuallyconvex} 
k''\leq 0 \enskip\mbox{on} \enskip(0,d) \qquad \mbox{and} \qquad k''\geq 0 \enskip\mbox{on} \enskip(d,1).
\end{equation}
By the Intermediate Value Theorem there exists an $x_1\in(x_0,m_g)$ such that $k(x_1) = k(m_g)/2$. Then
$0<x_0<x_1$ and
$$
k(x_0) \leq 0 < \frac{x_1-x_0}{x_1 - 0}k(0)+\frac{x_0 - 0}{x_1 - 0}k(x_1)
$$
so that $k$ cannot be concave on $(0,x_1)$. Hence, $k''$ takes on a positive value at some point in $(0,x_1)$. Therefore by (\ref{eq:eventuallyconvex}) we have $x_1\geq d$ and hence $k''\geq0$ on $(x_1,1)$.
Now by the Mean Value Theorem there exists some $x_2\in(x_1,m_g)$ such that
$$
k'(x_2) = \frac{k(m_g)-k(x_1)}{m_g-x_1}  = \frac{k(m_g)}{2(m_g-x_1)} \geq 0.
$$
Therefore $k'\geq k'(x_2)\geq 0$ on $(x_2,1)$ and therefore $k$ is increasing on $[m_g,1]$. In particular, $k(m_g) \leq k(1) = 0$, giving a contradiction and therefore proving our claim.

Finally, since $k>0$ on $(0,m_g)$ and $k(m_f)\leq 0$, we must therefore have $m_f\geq m_g$ as desired.   

\end{proof}

\begin{remark}
The converse to Corollary \ref{cor:movehump} is false as can be seen by taking, for example, $f(p) = p^5-5p^3+4p$ and $g(p)=p-p^2$.
\end{remark}

\section{Equivalence of Class Weighting to Transformation of the Impurity Function}\label{sec:weighting}
As mentioned in the introduction, a common way to bias a tree's construction toward performance on a specific class is by class weighting. As the previous section shows, another way to do this is
to choose an asymmetric impurity function to determine optimal splits. In this section we will see that class weighting gives rise to the exact same optimal splits as the optimal splits one obtains by transforming the
impurity function in a specific way. We will also see exactly how and when class weighting relates to the preceding section.

\begin{definition}\label{def:phi_w} 
For $w>0$, define $\phi_w:[0,1]\rightarrow[0,1]$ by
$$
\phi_w(p)\coloneqq \frac{wp}{1+(w-1)p}.
$$
\end{definition}

Suppose we have a node $n$ with positive prevalence $c$ and total weight $W$. Then $n$ has Class 0 weight equal to $W(1-c)$ and Class 1 weight equal to $Wc$.
If we transform $n$ into $\tilde n$ by scaling the weights of all Class 1 points in $n$ by a factor of $w$, then this transformed node $\tilde n$ still has Class 0 weight equal to $W(1-c)$ but now has Class 1 weight equal to $Wwc$, giving
$\tilde n$ an overall weight of $W(1-c)+Wwc=W(1+(w-1)c)$. The positive prevalence of $\tilde n$ is therefore equal to $Wwc/W(1+(w-1)c)=\phi_w(c)$. Now if the original unweighted node $n$ has a split into
children with positive prevalences $a$ and $b$, then similar reasoning as above shows that the children of the transformed node $\tilde n$ under the same split will have positive prevalences equal to $\phi_w(a)$ and $\phi_w(b)$. If we use preimpurity
function $f$ to determine node impurity, then this split of $\tilde n$ has total impurity equal to
$$
W(1+(w-1)c)\cdot\left(\frac{\phi_w(b)-\phi_w(c)}{\phi_w(b)-\phi_w(a)}\cdot f(\phi_w(a)) + \frac{\phi_w(c)-\phi_w(a)}{\phi_w(b)-\phi_w(a)} \cdot f(\phi_w(b))\right)
$$
by Proposition \ref{prop:childweights}.
Therefore, given a node $n$ with positive prevalence $c$, together with a collection $S$ of possible splits and a weighting factor $w$, the optimal split of the weighted node $\tilde n$ is given by
\begin{align*}
&\argmin_{(a,b)\in S}\left( W(1+(w-1)c)\cdot\left(\frac{\phi_w(b)-\phi_w(c)}{\phi_w(b)-\phi_w(a)}\cdot f(\phi_w(a)) + \frac{\phi_w(c)-\phi_w(a)}{\phi_w(b)-\phi_w(a)} \cdot f(\phi_w(b))\right)\right)\\
=&\argmin_{(a,b)\in S}\left(\frac{\phi_w(b)-\phi_w(c)}{\phi_w(b)-\phi_w(a)}\cdot f(\phi_w(a)) + \frac{\phi_w(c)-\phi_w(a)}{\phi_w(b)-\phi_w(a)} \cdot f(\phi_w(b))\right).
\end{align*}

\begin{definition}\label{def:T_w} 
Let $w>0$. Define the transformation $T_w$ on the set of functions $f$ on $[0,1]$ by
$$
(T_wf)(p) = (1+(w-1)p)\cdot(f\circ\phi_w)(p).
$$
\end{definition}

The preceding definitions and discussion put us in a position to quickly prove the first main theorem of this section:

\begin{theorem}\label{thm:weightequiv}
Let $f$ be a preimpurity function and $w>0$. Let $n$ be a node, and let $\tilde n$ be the node obtained from $n$ by scaling the weights of the Class 1 points by $w$. Then the optimal split of $\tilde n$ with respect
to $f$ is the same as the optimal split of $n$ with respect to $T_wf$. In other words: for every preimpurity function $f$, every $w>0$, every $c\in(0,1)$, and every $S\subseteq([0,c)\times(c,1])\cup\{(c,c)\}$ we have
\begin{align*}
&\argmin_{(a,b)\in S}\left(\frac{\phi_w(b)-\phi_w(c)}{\phi_w(b)-\phi_w(a)}\cdot f(\phi_w(a)) + \frac{\phi_w(c)-\phi_w(a)}{\phi_w(b)-\phi_w(a)} \cdot f(\phi_w(b))\right)\\
 = &\argmin_{(a,b)\in S}\left(\frac{b-c}{b-a}\,T_wf(a) + \frac{c-a}{b-a}\,T_wf(b)\right).
\end{align*}
\end{theorem}

\begin{proof}
Fix $f,w,c,S$ as above. Then a direct computation shows that for all $(a,b)\in S$ we have
\begin{align*}
&(1+(w-1)c)\cdot\left(\frac{\phi_w(b)-\phi_w(c)}{\phi_w(b)-\phi_w(a)}\  f(\phi_w(a)) + \frac{\phi_w(c)-\phi_w(a)}{\phi_w(b)-\phi_w(a)}\ f(\phi_w(b))\right)\\
=\ &(1+(w-1)c)\cdot\left(\frac{b-c}{b-a}\cdot\frac{1+(w-1)a}{1+(w-1)c}\cdot f(\phi_w(a)) + \frac{c-a}{b-a}\cdot\frac{1+(w-1)b}{1+(w-1)c}\cdot f(\phi_w(b))\right)\\
=\ &\frac{b-c}{b-a}\cdot(1+(w-1)a)\cdot f(\phi_w(a)) + \frac{c-a}{b-a}\cdot(1+(w-1)b)\cdot f(\phi_w(b))\\
=\ &\frac{b-c}{b-a}\,T_wf(a) + \frac{c-a}{b-a}\,T_wf(b).
\end{align*}
\end{proof}

\begin{remark}
The proof of Theorem \ref{thm:weightequiv} shows that not only are the optimal splits with respect to $T_wf$ the same as the optimal weighted splits with respect to $f$, but in fact by multiplying all of the above equations
by the total weight $W$ of $n$ we see that for {\it every} split the {\it value} of the impurity
of the split with respect to $T_wf$ is equal to the value of the impurity of the weighted split with respect to $f$.
\end{remark}

We now list some properties of $T_w$.

\begin{proposition}\label{prop:Tw}
Let $f,g$ be preimpurity functions, and let $w,w_1,w_2>0$. Then:
\begin{enumerate}
\item $T_w f$ is a preimpurity function. 

\item $T_{w_1}T_{w_2}=T_{w_1w_2}$.

\item $T_1=id$ and $T_{w}^{-1}=T_{1/w}$.

\item $f$ splits more positively purely than $g$ if and only if $T_wf$ splits more positively purely than $T_wg$.
\end{enumerate}
\end{proposition}

\begin{proof}

(1) Firstly, note that smoothness of $f$ is preserved since $T_wf$ is a precomposition and product of $f$ with smooth functions. Secondly, a direct computation shows  
\begin{equation}\label{eq:Twfconcave}
(T_wf)''(p) = \frac{w^2}{(1+(w-1)p)^3}(f''\circ\phi_w)(p)
\end{equation}
which is negative for $p\in(0,1)$ since $f''<0$, so strict concavity is preserved. So $T_w f$ is a preimpurity function.

(2),(3) These are direct computations and are left as an exercise to the reader.

(4) $(\Rightarrow)$ Suppose $f$ splits more positively purely than $g$. Then $f''/g''$ is increasing by Theorem \ref{thm:mainthm}. Equation (\ref{eq:Twfconcave}) above then gives
$$
\frac{(T_wf)''(p)}{(T_wg)''(p)} = \frac{\frac{w^2}{(1+(w-1)p)^3}(f''\circ\phi_w)(p)}{\frac{w^2}{(1+(w-1)p)^3}(g''\circ\phi_w)(p)} = \left(\frac{f''}{g''}\circ\phi_w\right)(p)
$$
which is increasing since $f''/g''$ and $\phi_w$ are increasing. So $T_wf$ splits more positively purely than $T_wg$ by Theorem \ref{thm:mainthm}.

$(\Leftarrow)$ Suppose $T_wf$ splits more positively purely than $T_wg$. Apply the forward implication of Part (4) to $T_wf$ and $T_wg$ using $T_{1/w}$ and Part (3).
\end{proof}

\begin{remark}
As it turns out, the family $h_m$ of functions in (\ref{eqn:MZRimp}) given in the introduction can be expressed in the form $T_wf$ (up to constant scaling) for some $f$. Specifically,
$$
h_m=\frac{1}{2(1-m)^2}\,T_wg
$$
where $w=(\frac{1}{m}-1)^2$ and $g$ is the Gini impurity. In other words, the tree produced by using the impurity function $h_m$ is the same as the tree produced by first weighting the Class 1 points by $(\frac{1}{m}-1)^2$
and then growing the tree using the Gini impurity.

Not every asymmetric impurity function $f$ is of the form $T_wg$ for some symmetric $g$. For example, let $f(p)=p-p^3$. If $f$ were of the form $T_wg$ for some symmetric $g$, then we would have $T_{1/w}f=g$, so that $T_{1/w}f$ is symmetric,
implying $(T_{1/w}f)''$ is symmetric. But this is never the case for any $w>0$ since $(T_{1/w}f)''(0)=0$ and $(T_{1/w}f)''(1)<0$.    
\end{remark}

Recall the plots shown in Figure \ref{fig:splits} in Section \ref{sec:performance}. For that specific example we saw that the Gini impurity after weighting the Class 1 points by a factor of 1/2 split more positively purely than the Gini impurity on the unweighted set, which in
turn split more positively purely than the Gini impurity after weighting the Class 1 points by a factor of 5. Indeed, this is an instance of a more general phenomenon, defined below.

\begin{definition}\label{def:respectweight}
Let $f$ be a preimpurity function. We say $f$ {\it respects class weighting} if for all $w_1,w_2>0$
$$
w_1\leq w_2 \Rightarrow T_{w_1}f\ \mbox{splits more positively purely than}\  T_{w_2}f.
$$
\end{definition}

The above condition can be rather messy to check as it potentially requires verifying that the inequality 
$$
\left(\frac{(T_{w_1}f)''}{(T_{w_2}f)''}\right)'(p) \geq 0
$$
holds for all appropriate values for the three quantities $p,w_1,w_2$. The following lemma allows us to reduce some of the computational messiness by eliminating one of the $w_i$.

\begin{lemma}\label{lem:simplerespect}
Let f be a preimpurity function. Then $f$ respects class weighting if and only if for all $w$ 
\begin{equation}\label{eq:simplerespect}
w\geq 1 \Rightarrow f\ \mbox{splits more positively purely than}\  T_wf.
\end{equation}
\end{lemma}

\begin{proof}
$(\Rightarrow)$ Let $w_1=1,w_2=w$ in Definition \ref{def:respectweight}.

$(\Leftarrow)$ Let $0<w_1\leq w_2$. Letting $w=w_2/w_1\geq1$ in (\ref{eq:simplerespect}) we get that $f$ splits more positively purely than $T_{w_2/w_1}f$. Applying Proposition \ref{prop:Tw}, Parts (2) and (4) using $T_{w_1}$ we get
$T_{w_1}f$ splits more positively purely than $T_{w_1}(T_{w_2/w_1}f) = T_{w_2}f$, as desired.
\end{proof}

In fact, we can fully characterize all preimpurity functions that respect class weighting (though we will need to impose an additional order of smoothness). This is the second main theorem of this section, and it ties together
Sections \ref{sec:mainsec} and \ref{sec:weighting}. To facilitate the presentation of the proof, we first list several equations whose proofs are direct computations and therefore omitted.

\begin{lemma}\label{lem:someeqns}
For all $w>0$ and $p\in(0,1)$ we have
\begin{align*}
\frac{1}{(1+(w-1)p)^2}&=\frac{(w-(w-1)\phi_w(p))^2}{w^2},\\
\frac{\partial}{\partial w}\phi_w(p)&=\frac{\phi_w(p)(1-\phi_w(p))}{w},\\
\phi_w'(p)&=\frac{(w-(w-1)\phi_w(p))^2}{w}, \quad and \\
\frac{\partial}{\partial w}\phi_w'(p)&=\frac{(1-2\phi_w(p))(w-(w-1)\phi_w(p))^2}{w^2}.
\end{align*}
\end{lemma}  

\begin{theorem}\label{thm:sectionlink}
Let $f$ be a preimpurity function, and suppose $f$ is $C^4$ on $(0,1)$. Let $H=\log(-f'')'=f'''/f''$, and define $G$ on $(0,1)$ by
$$
G(p)\coloneqq p(p-1)H'(p)+(2p-1)H(p)+3.
$$
Then $f$ respects class weighting if and only if $G\geq 0$.
\end{theorem}

\begin{proof}
First, observe that by Lemma \ref{lem:simplerespect} and Theorem \ref{thm:mainthm} we have
\begin{align*}
f\ \mbox{respects class weighting}\iff & \mbox{for all } w\geq 1 \quad f\ \mbox{splits more positively purely than}\  T_wf\\
\iff & \mbox{for all } w\geq 1 \quad \frac{f''}{(T_wf)''}\ \mbox{is increasing on }(0,1)\\
\iff & \mbox{for all } w\geq 1 \quad \log\left(\frac{f''}{(T_wf)''}\right)\ \mbox{is increasing on }(0,1)\\
\iff & \mbox{for all } w\geq 1 \quad \log\left(\frac{f''}{(T_wf)''}\right)' \geq 0 \mbox{ on }(0,1)\\
\iff & \mbox{for all } w\geq 1 \mbox{ and all } p\in (0,1)  \quad \log\left(\frac{f''}{(T_wf)''}\right)'(p) \geq 0.
\end{align*}
Define the function $F$ on $[1,\infty)\times(0,1)$ by
\begin{align*}
F(w,p) \coloneqq& \log\left(\frac{f''}{(T_wf)''}\right)'(p)\\
=& \log\left(\frac{(1+(w-1)p)^3\cdot f''(p)}{w^2 \cdot (f''\circ\phi_w)(p)}\right)'\\
=& \frac{f'''(p)}{f''(p)}+\frac{3(w-1)}{1+(w-1)p}-\frac{(f'''\circ\phi_w)(p)\cdot\phi_w'(p)}{(f''\circ\phi_w)(p)}
\end{align*}
where we used (\ref{eq:Twfconcave}) for the second equality.  
We therefore want to show $F \geq 0 \iff G\geq0$. Note that $F$ is $C^1$ by our hypothesis on $f$. We compute the partial derivative of $F$ with respect to $w$ and simplify using Lemma \ref{lem:someeqns}:
\begin{align*}
\frac{\partial F}{\partial w} (w,p) &= \frac{3}{(1+(w-1)p)^2}\\
&\quad - \left(\frac{\left[\frac{\partial}{\partial w}(f'''\circ\phi_w)(p)\cdot\phi_w'(p)+(f'''\circ\phi_w)(p)\cdot\frac{\partial}{\partial w}\phi_w'(p)\right]\cdot(f''\circ\phi_w)(p) }{(f''\circ\phi_w)(p)^2}\right.\\
& \hspace{1cm}- \left.\frac{(f'''\circ\phi_w)(p)\cdot\phi_w'(p) \cdot\frac{\partial}{\partial w} (f''\circ\phi_w)(p)}{(f''\circ\phi_w)(p)^2}\right)\\
&= \frac{3}{(1+(w-1)p)^2}\\
&\quad - \frac{(f^{(4)}\circ\phi_w)(p)\cdot\frac{\partial}{\partial w}\phi_w(p)\cdot \phi_w'(p) \cdot (f''\circ\phi_w)(p) }{(f''\circ\phi_w)(p)^2}\\
&\quad - \frac{(f'''\circ\phi_w)(p)\cdot\frac{\partial}{\partial w}\phi_w'(p)}{(f''\circ\phi_w)(p)}\\
&\quad + \frac{(f'''\circ\phi_w)(p)\cdot\phi_w'(p) \cdot (f'''\circ\phi_w)(p)\cdot\frac{\partial}{\partial w}\phi_w(p)}{(f''\circ\phi_w)(p)^2}\\
&= \frac{3}{(1+(w-1)p)^2}\\
&\quad - \frac{\partial}{\partial w}\phi_w(p)\cdot \phi_w'(p) \cdot\left(\frac{(f^{(4)}\circ\phi_w)(p)\cdot (f''\circ\phi_w)(p) - (f'''\circ\phi_w)(p) \cdot (f'''\circ\phi_w)(p)}{(f''\circ\phi_w)(p)^2}\right)\\
&\quad - \frac{(f'''\circ\phi_w)(p)\cdot\frac{\partial}{\partial w}\phi_w'(p)}{(f''\circ\phi_w)(p)}\\
&= \frac{3}{(1+(w-1)p)^2}\\
&\quad - \frac{\partial}{\partial w}\phi_w(p)\cdot \phi_w'(p) \cdot  \left(\frac{f'''}{f''}\right)'(\phi_w(p))\\
&\quad - \frac{\partial}{\partial w}\phi_w'(p)\cdot \left(\frac{f'''}{f''}\right)(\phi_w(p))\\
&= 3\frac{(w-(w-1)\phi_w(p))^2}{w^2}\\
&\quad - \frac{\phi_w(p)(1-\phi_w(p))}{w}\cdot \frac{(w-(w-1)\phi_w(p))^2}{w} \cdot H'(\phi_w(p))\\
&\quad - \frac{(1-2\phi_w(p))(w-(w-1)\phi_w(p))^2}{w^2}\cdot H(\phi_w(p))\\
&= \frac{(w-(w-1)\phi_w(p))^2}{w^2}\cdot\left(3- \phi_w(p)(1-\phi_w(p))\cdot H'(\phi_w(p) - (1-2\phi_w(p))\cdot H(\phi_w(p))\right)\\
&= \frac{(w-(w-1)\phi_w(p))^2}{w^2}\cdot G(\phi_w(p)).
\end{align*} 
In particular, evaluating at $w=1$ we get
$$
\frac{\partial F}{\partial w} (1,p) = G(p).
$$
Note also that $F(1,p)=0$ for all $p\in(0,1)$.
 
$(\Rightarrow)$ Now suppose $F\geq 0$. Then for every fixed $p\in (0,1)$ we have
\begin{align*}
\mbox{for all } w\geq 1 \quad F(w,p) \geq 0 \Rightarrow\ &\mbox{for all } w> 1 \quad \frac{F(w,p)-F(1,p)}{w-1}\geq 0\\
\Rightarrow\ & \lim_{w\rightarrow 1^+}\frac{F(w,p)-F(1,p)}{w-1}\geq 0\\
\Rightarrow\ & G(p) =\left.\frac{\partial}{\partial w}\right|_{w=1} F(w,p)\geq 0.
\end{align*}

$(\Leftarrow)$ Now suppose $G\geq 0$. Then for all $p,w$ we apply the Fundamental Theorem of Calculus and integrate over $w$ to get
\begin{align*}
F(w,p) &= F(1,p) + \int_1^w\frac{\partial F}{\partial w}(t,p)\,dt\\
&= 0 + \int_1^w\frac{(t-(t-1)\phi_t(p))^2}{t^2}\cdot G(\phi_t(p))\,dt\\
&\geq 0.
\end{align*}
\end{proof}

\begin{corollary}
Let $f$ be either the entropy or the Gini impurity. Then $f$ respects class weighting. 
\end{corollary}

\begin{proof}
For the cases of entropy and Gini impurity, we apply Theorem \ref{thm:sectionlink} and compute $G\equiv 1$ and $G\equiv 3$, respectively.
\end{proof}

\begin{remark}
Both of the cases of the entropy and Gini impurity respecting class weighting follow just as easily without Theorem \ref{thm:sectionlink} using Lemma \ref{lem:simplerespect}, Theorem \ref{thm:mainthm}, and (\ref{eq:Twfconcave}). Nevertheless, despite
the condition in Theorem \ref{thm:sectionlink} being somewhat messy, it is still an improvement over Lemma \ref{lem:simplerespect} in the sense that Theorem \ref{thm:sectionlink} reduces verification of Definition \ref{def:respectweight} to verification of nonnegativity
of a univariate function on the unit interval. 
\end{remark}
  
\begin{remark}
The impurity function $f(p)=p-p^3$ that we have been using in examples throughout this paper also respects class weighting, as do $f(p)=p-p^\alpha$ for $\alpha> 1$ and $f(p)=p^\alpha-p$ for $0<\alpha< 1$. In these cases, we apply Theorem \ref{thm:sectionlink}
and compute $G\equiv \alpha+1$.

For an example of a preimpurity function that does not respect class weighting, consider the preimpurity function (in fact, symmetric impurity function) $f(p) = 1 - 3(p-\frac{1}{2})^2 - 4(p-\frac{1}{2})^4$. Then using Theorem \ref{thm:sectionlink} we check that $G(1/2)<0$. Alternatively,
one can directly show that $f$ fails to split more positively purely than $T_2f$ using Theorem \ref{thm:mainthm}.
\end{remark}

Another very noteworthy example of an impurity function that respects class weighting is $f(p) = \sqrt{p(1-p)}$, considered in \cite{KM96} and shown there to satisfy certain error bounds. It was also shown in \cite{DH00} to be cost-insensitive, i.e., insensitive to class weighting.
For this particular $f$, we compute $T_wf = \sqrt{w}\cdot f$, so that $T_{w_1}f$ is actually equivalent to $T_{w_2}f$ for all $w_1,w_2$. In other words, class weighting doesn't change the optimal splits at all when using this impurity function.
This is indeed in agreement with \cite{DH00}.

In fact, we can revisit the proof of Theorem \ref{thm:sectionlink} to also characterize all cost-insensitive impurity functions. First, let us define cost-insensitivity in terms of the framework we have built so far:
\begin{definition}\label{def:costinsensitive}
Let $f$ be a preimpurity function. We say $f$ is {\it cost-insensitive} if $f$ is equivalent to $T_wf$ for all $w>0$.
\end{definition}

Now by Corollary \ref{cor:equivalence}, $f$ is cost-insensitive if and only if for all $w>0$ the function $f''/(T_wf)''$ is constant. Revisiting the definition of $F$ in the proof of Theorem \ref{thm:sectionlink}, we see that this is equivalent to $F\equiv 0$ on its domain. But this
is easily seen (again, by revisiting the proof of Theorem \ref{thm:sectionlink}) to be equivalent to $G\equiv 0$. In other words, $f$ is cost-insensitive if and only if $f$ satisfies the ODE
$$
p(p-1)H'(p)+(2p-1)H(p)+3=0
$$
where we recall $H=f'''/f''$. Now the solution to the above ODE is
$$
H(p)=\frac{3p+C_1}{p(1-p)}, \qquad p\in(0,1)
$$  
where $C_1$ is a constant. Since $H=f'''/f''=\log(-f'')'$ we integrate and exponentiate both sides of the above equality to obtain
\begin{align*}
f''(p) &= C_2\exp(C_1\log p - (C_1+3)\log(1-p))\\
&= C_2\cdot p^{C_1}(1-p)^{-C_1-3}.
\end{align*}
Integrating twice more and absorbing and relabeling constants we get
$$
f(p) = C_2\cdot p^{C_1+2}(1-p)^{-C_1-1} +C_3p+C_4.
$$
Requiring that our preimpurity function be continuous on the closed interval $[0,1]$ gives $-2<C_1<-1$. Imposing further that $f(0)=f(1)=0$ gives $C_3=C_4=0$. Finally, letting $C_2=1$ and $\alpha = C_1+2$ we get
$$
f(p) = p^\alpha(1-p)^{1-\alpha}, \qquad 0<\alpha<1.
$$
We have just proved the third and final main theorem of this section:
\begin{theorem}\label{thm:insensitive}
Let $f$ be an impurity function, and suppose $f$ is $C^4$ on $(0,1)$. Then $f$ is cost-insensitive if and only if $f$ is a positive scalar multiple of one of the functions in the family $\{f_\alpha\}$ given by
$$
f_\alpha(p)=p^\alpha(1-p)^{1-\alpha}, \qquad \alpha\in(0,1).
$$
\end{theorem}

\begin{remark}
A direct computation shows that for the above family we have $T_wf_\alpha=w^\alpha\cdot f_\alpha$, which is consistent with the backward implication in Theorem \ref{thm:insensitive}. Also, observe that for $\alpha,\beta\in(0,1)$ we compute
$$
\frac{f_\alpha''}{f_\beta''}(p)=\frac{\alpha(1-\alpha)}{\beta(1-\beta)}\cdot \left(\frac{1}{p}-1\right)^{\beta-\alpha}
$$
so that $f_\alpha$ splits more positively purely than $f_\beta$ if and only if $\alpha\geq\beta$. 
\end{remark}

\section{Some Remarks on the Axioms of Impurity Functions}\label{sec:impaxioms}

We conclude by summarizing some remarks made earlier in this paper on the axioms of an impurity function as typically given in the literature, stated at the top of Section \ref{sec:mainsec}. Recall those axioms:
\begin{enumerate} 
\item $f(p)$ is maximized only at $p=1/2$;
\item $f(p)$ is minimized only at the endpoints $p=0,1$;
\item $f$ is symmetric, i.e., $f(p)=f(1-p)$.
\end{enumerate}
As Corollary \ref{cor:standardform} shows, Axiom 2 is not necessary for good splitting behavior although there is no loss of generality in assuming Axiom 2. Furthermore,
even under the assumption that $f(0)=f(1)=0$, Axioms 1 and 3 are still not necessary for good splitting behavior; indeed, Theorem \ref{thm:weightequiv}
shows that asymmetric impurity functions are, in many cases, equivalent to symmetric impurity functions under class weighting.

The one property we did emphasize in our definition of impurity function is concavity. Indeed, while concavity is not explicitly stated as one of the axioms of an impurity function above, strict concavity is
typically additionally imposed upon (or implicitly satisfied by) the impurity functions under consideration. The reason for this is to ensure that total impurity is decreased by splitting a node \cite{BFOS84}. For completeness,
we present a full argument below.

Consider the following example. Let $f(p) = p^4(1-p)^4$. Then $f$ satisfies
Axioms 1-3 but is not concave. Now place two points of Class 0 and one point of Class 1, each with unit weight, on the real line in the order `010'. Then the impurity of this set is $3f(1/3) = 16/2187\approx .0073$. But the two nontrivial splits \{`01',`0'\} and \{`0',`10'\}
each have impurity equal to $1f(0)+2f(1/2) = 1/128\approx .0078$, giving an {\it increase} in impurity, causing our node to become ``stuck" and unable to split.

A property that an impurity function ought to have is that making a split should never increase total impurity; or, using the entropy/information gain heuristic, one should never lose information by splitting a node. We state this precisely below:

\begin{definition}
We say a function $f$ on $[0,1]$ is {\it proper} if for every node $n$ and every split of $n$, the total impurity of that split with respect to $f$ is less than or equal to the impurity of $n$ with respect to $f$. In other words, $f$ is proper if
for all $c\in(0,1)$ and all $(a,b)\in([0,c)\times(c,1])\cup\{(c,c)\}$ we have
$$
\frac{b-c}{b-a} f(a) + \frac{c-a}{b-a} f(b) \leq f(c).
$$
\end{definition}

With this definition it is easy to see that the property of being proper is just a slight rephrasing of concavity, making the following proposition immediate:
\begin{proposition}
$f$ is proper if and only if $f$ is concave.
\end{proposition}
One usually also desires that the impurity function should be nondegenerate in the sense that impurity should strictly decrease (i.e., information gain should be positive) if the split is nontrivial, i.e., $a<c<b$. This is easily seen to be equivalent to strict concavity of $f$.

\end{document}